\documentclass[11pt]{article}

\usepackage[margin=0.92in]{geometry}
\usepackage[T1]{fontenc}
\usepackage[utf8]{inputenc}
\usepackage{lmodern}
\usepackage{microtype}
\usepackage{amsmath,amssymb,amsthm}
\usepackage{booktabs}
\usepackage{tabularx}
\usepackage{graphicx}
\usepackage{placeins}
\usepackage{float}
\usepackage{enumitem}
\usepackage[numbers,sort&compress]{natbib}
\usepackage{xurl}

\usepackage[hidelinks]{hyperref}

\newtheorem{proposition}{Proposition}[section]
\newtheorem{lemma}[proposition]{Lemma}
\newtheorem{corollary}[proposition]{Corollary}
\newcommand{\loss}{\mathcal{L}}
\newcommand{\feasible}{\mathcal{F}}
\newcommand{\pool}{\mathcal{R}}
\newcommand{\E}{\mathbb{E}}

\newcolumntype{Y}{>{\raggedright\arraybackslash}X}

\title{\textbf{PruneShift: A Framework for Evaluating Decision\\
Reliability in Structured Pruning}}
\author{%
  Hao Ye$^{1,2}$ and Gaopeng Zhang$^{1}$\\
  \small $^{1}$Xi'an Institute of Optics and Precision Mechanics, Chinese Academy of Sciences\\
  \small $^{2}$University of Chinese Academy of Sciences\\
  \small \texttt{yehao22@mails.ucas.ac.cn}; \texttt{zhanggaopeng@opt.ac.cn}%
}

\date{}

\begin{document}
\maketitle

\begin{abstract}
Structured pruning uses surrogate objectives because direct task evaluation
over every feasible mask is too expensive. Most evaluations report average
surrogate error or rank correlation on broadly sampled masks. These summaries
do not directly test the mask chosen by the surrogate. We introduce
PruneShift, an evaluation framework that separates broad predictive fidelity,
fidelity near selector outputs, and the quality of the selected pruning
decision.
We first prove that Spearman and Kendall agreement can approach one while
normalized selection regret remains maximal. We then derive sufficient
conditions based on uniform error, selector suboptimality, decision margin,
density ratio, and comparison mass. The analysis also yields a finite pool
certificate with an explicit excess cost bound. Four studies test
different links in this argument. External
TextbookQA confirmation is heterogeneous: 7 of 20 simultaneous intervals
favor the surrogate-selected mask, 6 favor its fixed comparator, and 7 cross
zero. On a fixed Natural Questions pool, strict improvement holds in one of
four settings. A controlled QQP experiment supports the proposed coverage
mechanism in all 16 prespecified endpoints, although the sufficient bounds are
conservative. Finally, a restricted OSSCAR reconstruction study on OPT-125M
shows better local than broad fidelity in 68 of 75 primary endpoints.
Independent fixed-mask confirmation is inconclusive in 24 of 25 endpoints and
favors the comparator in one. These results show why predictive fit, decision
reliability, and pruning method quality require separate evidence.
\end{abstract}

\section{Introduction}

Structured pruning removes complete attention heads, channels, neurons, or
blocks. The resulting model can retain regular tensor structure, so this form
of sparsity is attractive for deployment
\citep{vaswani2017attention,michel2019sixteen,lagunas2021block,xia2022cofi}.
The search problem is combinatorial. A typical method therefore builds a
surrogate from weights, activations, gradients, curvature, or reconstruction
error. It then optimizes that surrogate under a pruning budget
\citep{lecun1990obd,hassibi1993obs,nonnenmacher2022sosp,
kwon2022fast,meng2024osscar,gong2026cocurve}.

This workflow creates an evaluation gap. The surrogate is often checked on a
broad sample of masks. The selector, however, searches for masks with unusually
favorable surrogate values. An error that is rare under broad sampling can be
common among the masks that compete for selection. Average error and rank
correlation can still look strong. The selected mask can nevertheless be poor.

The problem is not unique to pruning. Decision-focused learning evaluates a
prediction through the choice it induces. Offline optimization studies how a
selector can exploit errors in weakly supported regions. The optimizer's curse
describes a related selection effect
\citep{smith2006optimizer,wilder2019decision,elmachtoub2022smart,
trabucco2021conservative,tan2025offline,tsiourvas2024curse}. Structured pruning
adds a useful distinction. A surrogate estimator can fail even when its search
algorithm works as intended. A selector can also fail even when the surrogate
is accurate on the relevant masks. These mechanisms require different tests.

PruneShift is an evaluation framework, not a new pruning score. It asks three
questions. Is the surrogate accurate across generic feasible masks? Is it
accurate near the masks reached by selection? Does the selected decision perform
well against declared alternatives on independent data? We call these the
broad, selector-neighborhood, and comparison domains. Figure
\ref{fig:framework-theory} summarizes the separation.

The paper makes four contributions.

\begin{enumerate}[leftmargin=1.55em,label=\arabic*.]
  \item We define a decision reliability evaluation for pruning surrogates. It separates broad
        prediction, selector-neighborhood prediction, and finite comparison
        regret. It also separates estimator error from selector error.
  \item We prove an inversion between rank agreement and decision quality. We
        then derive sufficient guarantees using coverage, comparison mass,
        uniform error, selector suboptimality, and decision margin.
  \item We give two operational routes to stronger evidence. A fixed-look
        finite-pool certificate bounds excess cost within a fixed pool.
        Independent confirmation directly tests a prespecified comparison. These
        routes answer different questions and are not presented as substitutes.
  \item We test the framework through external transport, finite-pool
        selection, a controlled coverage intervention, and a public OSSCAR
        reconstruction study. The results contain both positive and null
        findings. Together they show which claims do and do not transfer.
\end{enumerate}

\section{Related Work}

Structured pruning began with local sensitivity and curvature estimates
\citep{lecun1990obd,hassibi1993obs}. Later methods learned masks during
training or used sparsity signals at initialization
\citep{frankle2019lottery,sanh2020movement}. Transformer pruning has removed
attention heads, feed-forward groups, and larger blocks
\citep{michel2019sixteen,voita2019heads,lagunas2021block,xia2022cofi}.
Recent methods use global objectives, second-order information,
reconstruction, or discrete optimization
\citep{nonnenmacher2022sosp,kwon2022fast,kurtic2023ziplm,
vanderouderaa2024surgeon,meng2024osscar,wang2026grasprune,
huang2026ddp,wang2026global}. This literature mainly improves the pruning rule. PruneShift
instead studies the evidence needed to trust the rule's selected mask.

Structured units interact across heads, layers, and feed-forward groups.
Curvature, Shapley values, Harsanyi interactions, and quadratic objectives
represent some of these effects
\citep{held2023shapley,qu2025harsanyi,gong2026cocurve}. Head importance and
prunability can also diverge \citep{budhraja2020weak}. Search introduces a
second source of variation. Greedy construction and local exchange can return
different masks under the same objective \citep{zimmer2025sparseswaps}. Our
experiments therefore cross matched estimators and selectors. We use these
methods as evaluation subjects, not as novelty claims.

Decision-focused learning measures a prediction through the decision it
supports \citep{wilder2019decision,elmachtoub2022smart,mandi2022decision}.
Surrogate-based combinatorial optimization and contextual optimization use a
similar distinction
\citep{ferber2023surco,rosenfeld2018optimize,bennouna2025contextual}.
Offline model optimization further shows how selection can exploit prediction
errors in regions with weak support
\citep{trabucco2021conservative,tan2025offline}. Recent evaluation work has
also paired predictive fit with regret or top-decision quality
\citep{heuton2025decision,cieslak2026decision}. PruneShift brings this view to
structured pruning. Its main addition is an explicit separation between broad
masks, selector neighborhoods, and prespecified comparisons, together with tests for
the transfer between them.

\section{Problem Formulation and Evaluation Domains}

\subsection{Structured pruning as surrogate optimization}

Let \(z\in\feasible\subseteq\{0,1\}^{d}\) denote a feasible structured mask.
The set \(\feasible\) fixes the pruning cardinality and any layer constraints.
For a data distribution \(P_X\), the population loss is
\begin{equation}
  \loss(z)=\E_{X\sim P_X}[\ell(X;z)].
\end{equation}
For a finite split \(D\), we write the corresponding empirical mean as
\begin{equation}
  \bar{\loss}_{D}(z)
  =
  \frac{1}{|D|}\sum_{x\in D}\ell(x;z).
\end{equation}
Direct evaluation of \(\loss(z)\) for every \(z\in\feasible\) is usually
intractable. A pruning method constructs a surrogate
\(\widehat{\loss}(z)\) and returns
\begin{equation}
  \widehat z\in\arg\min_{z\in\feasible}\widehat{\loss}(z).
\end{equation}
The estimator defines \(\widehat{\loss}\). The selector determines how its
constrained minimum is approximated. We treat them as separate components.

For any \(A\subseteq\feasible\) that contains \(\widehat z\), define the
decision regret
\begin{equation}
  R_A(\widehat z)
  =
  \loss(\widehat z)-\min_{z\in A}\loss(z).
  \label{eq:domain-regret}
\end{equation}
When \(A=\feasible\), this is global regret. In experiments, \(A\) is a
finite prespecified comparison domain. The resulting regret is directly
estimable from held out data but deliberately scoped.

\subsection{Three non-interchangeable domains}

PruneShift distinguishes the following objects.

\begin{description}[leftmargin=1.2em,style=nextline]
  \item[Broad validation distribution \(Q_{\mathrm U}\).]
  This distribution samples generic feasible masks. It supports mean error,
  rank correlation, tail recall, and pairwise ordering.
  \item[Selector neighborhood distribution \(Q_{\mathrm N}\).]
  This distribution balances neighborhoods around fixed selector outputs.
  It measures the errors that search is more likely to expose.
  \item[Decision comparison domain \(C_{\mathrm S}\).]
  This finite set contains a selected mask and declared alternatives.
  Evaluating task loss on this set measures the actual decision in a known
  domain.
\end{description}

For two fixed cardinality masks with selected unit sets \(S\) and \(T\), we
use the half Hamming distance
\begin{equation}
  d_{\mathrm{hh}}(S,T)
  =
  \frac12|S\triangle T|
  =
  |S\setminus T|.
\end{equation}
The encoder neighborhoods use shells of radii \(1,2,3\) and assign equal
weight to every center and radius. This prevents a large neighborhood from
dominating a local summary.

The pointwise surrogate error is
\begin{equation}
  e(z)=\widehat{\loss}(z)-\loss(z).
\end{equation}
An average of \(e(z)^2\) under \(Q_{\mathrm U}\) is not automatically an
average under \(Q_{\mathrm N}\). An average under either distribution is not
automatically a uniform error bound on \(C_{\mathrm S}\). The next section
gives sufficient conditions for each transfer.

\begin{figure}[!tbp]
\centering
\includegraphics[width=\textwidth]{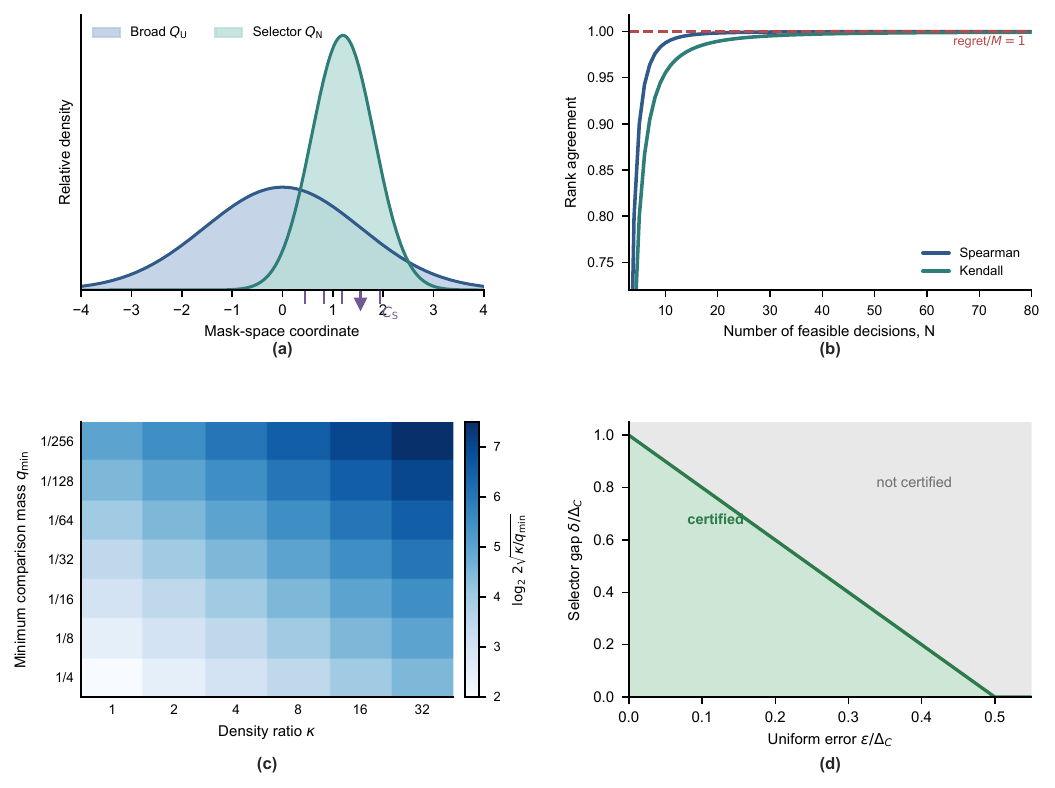}
\caption{\textbf{PruneShift separates predictive fidelity from decision
quality.}
\textbf{(a)} The broad distribution \(Q_{\mathrm U}\), selector-neighborhood
distribution \(Q_{\mathrm N}\), and fixed comparison set \(C_{\mathrm S}\)
define different evaluation domains. Transfer requires explicit coverage and
comparison mass.
\textbf{(b)} In Proposition~\ref{prop:inversion}, Spearman and Kendall
agreement approach one as the number of decisions grows, while normalized
regret remains one.
\textbf{(c)} The exact multiplier in Proposition
\ref{prop:finite-transfer} increases with the density ratio \(\kappa\) and
with smaller comparison mass \(q_{\min}\).
\textbf{(d)} The shaded region satisfies the exact-recovery condition after
normalizing uniform error and selector gap by the decision margin. Panels
\textbf{(a)}, \textbf{(c)}, and \textbf{(d)} are analytic illustrations rather
than empirical measurements.}
\label{fig:framework-theory}
\end{figure}

\section{Theory: From Fidelity to a Decision Guarantee}
\label{sec:theory}

\subsection{A decisive inversion}

\begin{proposition}[Near perfect rank agreement with arbitrary regret]
\label{prop:inversion}
For every \(M>0\) and every \(N\geq3\), there are true losses and surrogate
scores on \(N\) feasible decisions such that surrogate minimization incurs
regret \(M\), while Spearman correlation is
\begin{equation}
  \rho
  =
  1-\frac{12}{N(N^2-1)},
\end{equation}
Kendall correlation is
\begin{equation}
  \tau
  =
  1-\frac{4}{N(N-1)},
\end{equation}
and root mean squared error is \(M\sqrt{2/N}\).
\end{proposition}

\begin{proof}
Assign the two best true decisions losses \(0\) and \(M\). Assign all
remaining decisions distinct losses larger than \(M\). Let the surrogate swap
only the first two values and leave the other \(N-2\) values unchanged. The
surrogate therefore selects the decision with true loss \(M\).

The two swapped ranks each move by one position, so the sum of squared rank
displacements is \(2\). Spearman correlation is consequently
\begin{equation}
  1-\frac{6(2)}{N(N^2-1)}
  =
  1-\frac{12}{N(N^2-1)}.
\end{equation}
Exactly one of the \(N(N-1)/2\) pairs is discordant. Kendall correlation is
\begin{equation}
  1-\frac{2}{N(N-1)/2}
  =
  1-\frac{4}{N(N-1)}.
\end{equation}
The construction has no ties, so the usual \(\tau_a\) and \(\tau_b\)
definitions coincide.
Only two score errors are nonzero, with magnitudes \(M\) and \(M\). Hence
\begin{equation}
  \operatorname{RMSE}
  =
  \sqrt{\frac{M^2+M^2}{N}}
  =
  M\sqrt{\frac{2}{N}}.
\end{equation}
The selected decision has regret \(M\). For fixed \(M\), both rank
disagreement and RMSE vanish as \(N\) grows, whereas regret does not.
\end{proof}

Proposition~\ref{prop:inversion} is an insufficiency result. It does not say
that correlation is useless. It says that aggregate fidelity needs a
coverage or separation condition before it can certify an argmin.

\subsection{Transferring error across mask distributions}

\begin{proposition}[Density ratio error transfer]
\label{prop:density-transfer}
Let \(P\) and \(Q\) be distributions on \(\feasible\), with \(Q\ll P\).
Suppose
\begin{equation}
  \frac{dQ}{dP}\leq\kappa
  \qquad P\text{-almost surely}.
\end{equation}
Then
\begin{equation}
  \E_Q[e(z)^2]\leq\kappa\,\E_P[e(z)^2]
\end{equation}
and
\begin{equation}
  \E_Q[|e(z)|]
  \leq
  \sqrt{\kappa\,\E_P[e(z)^2]}.
\end{equation}
\end{proposition}

\begin{proof}
Let \(w=dQ/dP\). Change of measure gives
\begin{equation}
  \E_Q[e^2]
  =
  \E_P[we^2]
  \leq
  \kappa\E_P[e^2].
\end{equation}
Cauchy--Schwarz gives
\begin{equation}
  \E_Q|e|
  \leq
  \sqrt{\E_Q[e^2]}
  \leq
  \sqrt{\kappa\E_P[e^2]}.
\end{equation}
\end{proof}

Take \(P=Q_{\mathrm U}\) and \(Q=Q_{\mathrm N}\). The coefficient
\(\kappa\) quantifies the shift from broad validation to selector
neighborhoods. If \(Q_{\mathrm U}\) is uniform on \(N\) masks and
\(Q_{\mathrm N}\) is a point mass, then \(\kappa=N\). A single tail error can
therefore receive an \(N\)-fold amplification in mean squared error.

\subsection{Uniform error, margin, and finite comparison regret}

Let \(C\subseteq\feasible\), and define
\begin{equation}
  z_C^\star\in\arg\min_{z\in C}\loss(z),
  \qquad
  \widehat z_C\in\arg\min_{z\in C}\widehat{\loss}(z).
\end{equation}
For \(\delta\geq0\), a feasible output \(\widetilde z_C\in C\) is a
\(\delta\)-approximate surrogate minimizer if
\begin{equation}
  \widehat{\loss}(\widetilde z_C)
  \leq
  \min_{z\in C}\widehat{\loss}(z)+\delta.
  \label{eq:delta-selector}
\end{equation}
The quantity \(\delta\) measures selector optimization error in the same units
as the surrogate. It is distinct from the estimator error \(e(z)\).

\begin{lemma}[Uniform comparison error]
\label{lem:uniform-error}
If
\begin{equation}
  \sup_{z\in C}|e(z)|\leq\varepsilon,
\end{equation}
then
\begin{equation}
  \loss(\widehat z_C)-\loss(z_C^\star)\leq2\varepsilon.
\end{equation}
\end{lemma}

\begin{proof}
Surrogate optimality and the uniform error bound imply
\begin{equation}
\begin{aligned}
  \loss(\widehat z_C)
  &\leq
  \widehat{\loss}(\widehat z_C)+\varepsilon\\
  &\leq
  \widehat{\loss}(z_C^\star)+\varepsilon\\
  &\leq
  \loss(z_C^\star)+2\varepsilon.
\end{aligned}
\end{equation}
\end{proof}

\begin{proposition}[Estimator and approximate selector error]
\label{prop:approx-selector}
If Equation~\ref{eq:delta-selector} holds and
\begin{equation}
  \sup_{z\in C}|e(z)|\leq\varepsilon,
\end{equation}
then
\begin{equation}
  \loss(\widetilde z_C)-\loss(z_C^\star)
  \leq
  \delta+2\varepsilon.
  \label{eq:delta-uniform-bound}
\end{equation}
\end{proposition}

\begin{proof}
Feasibility, approximate surrogate optimality, and the uniform error bound give
\begin{equation}
\begin{aligned}
  \loss(\widetilde z_C)
  &\leq
  \widehat{\loss}(\widetilde z_C)+\varepsilon\\
  &\leq
  \min_{z\in C}\widehat{\loss}(z)+\delta+\varepsilon\\
  &\leq
  \widehat{\loss}(z_C^\star)+\delta+\varepsilon\\
  &\leq
  \loss(z_C^\star)+\delta+2\varepsilon.
\end{aligned}
\end{equation}
\end{proof}

\begin{corollary}[Margin safe selection]
\label{cor:margin}
Suppose \(z_C^\star\) is unique and
\begin{equation}
  \Delta_C
  =
  \min_{z\in C\setminus\{z_C^\star\}}
  \bigl(\loss(z)-\loss(z_C^\star)\bigr)
  >
  2\varepsilon.
\end{equation}
Then every surrogate minimizer on \(C\) equals \(z_C^\star\).
\end{corollary}

\begin{proof}
For any \(z\neq z_C^\star\),
\begin{equation}
  \widehat{\loss}(z)-\widehat{\loss}(z_C^\star)
  \geq
  \loss(z)-\loss(z_C^\star)-2\varepsilon
  >
  0.
\end{equation}
Thus \(z_C^\star\) is also the unique surrogate minimizer.
\end{proof}

\begin{corollary}[Margin safe approximate selection]
\label{cor:approx-margin}
Suppose \(z_C^\star\) is unique, Equation~\ref{eq:delta-selector} holds, and
\begin{equation}
  \Delta_C>\delta+2\varepsilon.
\end{equation}
Then \(\widetilde z_C=z_C^\star\).
\end{corollary}

\begin{proof}
For every \(z\neq z_C^\star\),
\begin{equation}
  \widehat{\loss}(z)-\widehat{\loss}(z_C^\star)
  \geq
  \loss(z)-\loss(z_C^\star)-2\varepsilon
  >
  \delta.
\end{equation}
Thus \(z_C^\star\) is the surrogate minimizer and every other point lies more
than \(\delta\) above it. No other point can satisfy Equation
\ref{eq:delta-selector}.
\end{proof}

The expectation to uniform step requires a second coverage quantity. Let
\(Q\) be a distribution on finite \(C\) with \(Q(z)>0\) for every \(z\in C\),
and define
\begin{equation}
  q_{\min}=\min_{z\in C}Q(z)>0.
\end{equation}

\begin{proposition}[Finite comparison transfer]
\label{prop:finite-transfer}
Under the density ratio condition of Proposition
\ref{prop:density-transfer},
\begin{equation}
  \loss(\widehat z_C)-\loss(z_C^\star)
  \leq
  2\sqrt{
    \frac{\kappa\,\E_P[e(z)^2]}{q_{\min}}
  }.
  \label{eq:coverage-bound}
\end{equation}
\end{proposition}

\begin{proof}
For every \(z\in C\),
\begin{equation}
  Q(z)e(z)^2
  \leq
  \sum_{u\in C}Q(u)e(u)^2
  =
  \E_Q[e^2].
\end{equation}
Because \(Q(z)\geq q_{\min}\),
\begin{equation}
  \sup_{z\in C}|e(z)|
  \leq
  \sqrt{\frac{\E_Q[e^2]}{q_{\min}}}
  \leq
  \sqrt{\frac{\kappa\E_P[e^2]}{q_{\min}}}.
\end{equation}
Applying Lemma~\ref{lem:uniform-error} proves
Equation~\ref{eq:coverage-bound}.
\end{proof}

\begin{corollary}[Finite comparison transfer for an approximate selector]
\label{cor:approx-finite-transfer}
Under the conditions of Proposition~\ref{prop:finite-transfer}, any
\(\delta\)-approximate surrogate minimizer satisfies
\begin{equation}
  \loss(\widetilde z_C)-\loss(z_C^\star)
  \leq
  \delta+
  2\sqrt{
    \frac{\kappa\,\E_P[e(z)^2]}{q_{\min}}
  }.
  \label{eq:delta-coverage-bound}
\end{equation}
\end{corollary}

\begin{proof}
The proof of Proposition~\ref{prop:finite-transfer} establishes
\begin{equation}
  \sup_{z\in C}|e(z)|
  \leq
  \sqrt{\frac{\kappa\E_P[e(z)^2]}{q_{\min}}}.
\end{equation}
Substitution into Proposition~\ref{prop:approx-selector} gives Equation
\ref{eq:delta-coverage-bound}.
\end{proof}

The two factors have different meanings. A large \(\kappa\) indicates that
search concentrates where broad validation has little mass. A small
\(q_{\min}\) indicates that a decisive comparison candidate receives little
evaluation weight.

\begin{corollary}[Controlled mixture design]
\label{cor:mixture}
Let \(C\) be a finite comparison domain, let \(Q_q\) be supported on \(C\)
with minimum mass \(q_{\min}\), and let \(U\) have support disjoint from
\(C\). For
\begin{equation}
  P_{\lambda,q}
  =
  (1-\lambda)U+\lambda Q_q,
  \qquad 0<\lambda\leq1,
\end{equation}
the density ratio from \(P_{\lambda,q}\) to \(Q_q\) is
\(\kappa=1/\lambda\). Therefore
\begin{equation}
  R_C(\widehat z_C)
  \leq
  2\sqrt{
    \frac{\E_{P_{\lambda,q}}[e(z)^2]}
         {\lambda q_{\min}}
  }.
\end{equation}
\end{corollary}

\begin{proof}
For every \(z\in C\), disjoint support gives
\begin{equation}
  P_{\lambda,q}(z)=\lambda Q_q(z).
\end{equation}
Hence \(Q_q(z)/P_{\lambda,q}(z)=1/\lambda\) on \(C\).
Substitute \(\kappa=1/\lambda\) into Proposition
\ref{prop:finite-transfer}.
\end{proof}

For a \(\delta\)-approximate selector, Corollary
\ref{cor:approx-finite-transfer} gives the corresponding bound
\begin{equation}
  R_C(\widetilde z_C)
  \leq
  \delta+
  2\sqrt{
    \frac{\E_{P_{\lambda,q}}[e(z)^2]}
         {\lambda q_{\min}}
  }.
\end{equation}

When \(\lambda=0\), the comparison domain has no support under the validation
distribution and the density ratio is unbounded. This limiting case is a
designed failure of the sufficient condition, not a claim about the behavior
of every possible estimator.

\subsection{Operational scope of the sufficient conditions}

The preceding transfer results are population statements conditional on fixed
objects. The functions \(\loss\) and \(\widehat{\loss}\) must be finite and
measurable under the stated distributions. The comparison domain \(C\) must be
nonempty and finite, which guarantees that its minima exist. The evaluation
distribution \(Q\) must give every member of \(C\) positive mass, \(Q\ll P\)
must hold, and the displayed \(\kappa\) must be a valid essential upper bound
on \(dQ/dP\). These conditions cannot be inferred from broad test error alone.

If the surrogate was fitted from calibration data, these statements are
conditional on the fitted function. Likewise, \(\E_P[e^2]\) in Equations
\ref{eq:coverage-bound} and \ref{eq:delta-coverage-bound} is a population
mean, not its empirical plug-in estimate. Turning these equations into a
finite-sample certificate would additionally require an
independent evaluation sample and a valid upper confidence bound for this mean,
together with known or conservatively bounded values of \(\kappa\),
\(q_{\min}\), and \(\delta\). The value \(\delta\) must certify objective
suboptimality relative to the minimum of the same surrogate on \(C\).
Termination at a one-swap local optimum does not by itself supply such a
domain-wide certificate. Consequently, the natural-data studies below diagnose
decision behavior but do not report a numerical population guarantee from
Equation~\ref{eq:delta-coverage-bound}. The controlled coverage experiment
fixes \(P\), \(Q\), and \(C\) so that \(\kappa\) and \(q_{\min}\) are known;
it tests the predicted mechanism rather than claiming that the population
upper bound is tight.

For a fixed finite pool, an independent evaluation sample gives a direct bridge
that avoids estimating a natural density ratio. Let
\(\ell(z;\xi)\in[a,a+B]\), let
\(\loss(z)=\E[\ell(z;\xi)]\), and define the evaluation mean
\begin{equation}
  \overline\loss_n(z)=\frac1n\sum_{i=1}^n\ell(z;\xi_i).
\end{equation}
The candidate set and surrogate below are fixed before the independent and
identically distributed observations \(\xi_1,\ldots,\xi_n\) are used.

\begin{proposition}[Independent finite-pool certificate]
\label{prop:finite-pool-certificate}
Let \(|C|=K\), and set
\begin{equation}
  r_{n,K,\alpha}
  =
  B\sqrt{\frac{\log(2K/\alpha)}{2n}},
  \qquad
  \widehat\varepsilon_n
  =
  \max_{z\in C}|\widehat{\loss}(z)-\overline\loss_n(z)|.
\end{equation}
With probability at least \(1-\alpha\), every \(\delta\)-approximate
surrogate minimizer satisfies
\begin{equation}
  \loss(\widetilde z_C)-\loss(z_C^\star)
  \leq
  \delta+2\bigl(\widehat\varepsilon_n+r_{n,K,\alpha}\bigr).
  \label{eq:finite-pool-certificate}
\end{equation}
\end{proposition}

\begin{proof}
Hoeffding's inequality and a union bound over the \(K\) candidates give,
with probability at least \(1-\alpha\),
\begin{equation}
  \max_{z\in C}|\overline\loss_n(z)-\loss(z)|
  \leq r_{n,K,\alpha}.
\end{equation}
On this event, the triangle inequality yields
\begin{equation}
  \max_{z\in C}|\widehat{\loss}(z)-\loss(z)|
  \leq \widehat\varepsilon_n+r_{n,K,\alpha}.
\end{equation}
Proposition~\ref{prop:approx-selector} completes the proof.
\end{proof}

This proposition certifies only the declared finite pool. It requires the
surrogate, comparison set, selector rule, and \(\delta\) certificate to be
fixed independently of the evaluation observations. It neither proves
coverage of the full feasible space nor repairs adaptive reuse of an already inspected
confirmation set.

\subsection{What a finite comparison set can establish}

\begin{lemma}[Finite pool regret is a lower bound]
\label{lem:pool-lower}
Let \(\pool\subseteq\feasible\) and \(\widehat z\in\pool\). Define
\begin{equation}
  R_{\mathrm{pool}}
  =
  \loss(\widehat z)-\min_{z\in\pool}\loss(z)
\end{equation}
and
\begin{equation}
  R_{\mathrm{global}}
  =
  \loss(\widehat z)-\min_{z\in\feasible}\loss(z).
\end{equation}
Then
\begin{equation}
  0\leq R_{\mathrm{pool}}\leq R_{\mathrm{global}}.
\end{equation}
\end{lemma}

\begin{proof}
Because \(\pool\subseteq\feasible\),
\begin{equation}
  \min_{z\in\feasible}\loss(z)
  \leq
  \min_{z\in\pool}\loss(z).
\end{equation}
Subtracting both sides from \(\loss(\widehat z)\) gives the upper inequality.
The lower inequality follows because \(\widehat z\in\pool\).
\end{proof}

Positive pool regret disproves global optimality. Zero pool regret does not
establish global optimality, and the pool value is not an upper bound on the
unobserved global gap.

\subsection{A fixed look finite pool certificate}

Suppose a pool of \(K\) candidates is fixed before the selection sample is
observed. Let
\(Z_e=(Z_{1,e},\ldots,Z_{K,e})\) be independent across examples \(e\). For
each candidate \(i\), assume \(Z_{i,e}\in[0,1]\) and
\(\E[Z_{i,e}]=\mu_i\). Coordinates of \(Z_e\) may be dependent, so the proof
does not require candidate independence within an example. At \(J\) fixed
cumulative sample sizes, define
\begin{equation}
  r(n)
  =
  \sqrt{\frac{\log(2KJ/\alpha)}{2n}}.
  \label{eq:hoeffding-radius}
\end{equation}
At a look with \(n\) examples, candidate \(i\) has interval
\begin{equation}
  I_i(n)
  =
  [\widehat\mu_i(n)-r(n),\widehat\mu_i(n)+r(n)].
\end{equation}
At every nonfinal look, eliminate \(i\) only if its lower endpoint exceeds
the smallest upper endpoint among active candidates.

\begin{proposition}[Elimination safety and final pool excess]
\label{prop:finite-pool}
With probability at least \(1-\alpha\), the fixed look procedure satisfies:
\begin{enumerate}[leftmargin=1.5em]
  \item at least one population optimal candidate is never eliminated; and
  \item selecting the smallest final empirical mean gives
  \begin{equation}
    \mu_{\widehat i}-\min_i\mu_i
    \leq
    2r(n_{\mathrm{final}}).
  \end{equation}
\end{enumerate}
\end{proposition}

\begin{proof}
For each candidate and each fixed look, Hoeffding's inequality gives
\begin{equation}
  \Pr\!\left(
    |\widehat\mu_i(n)-\mu_i|>r(n)
  \right)
  \leq
  2\exp(-2nr(n)^2)
  =
  \frac{\alpha}{KJ}.
\end{equation}
A union bound over \(KJ\) intervals gives a simultaneous coverage event
\(\mathcal E\) with probability at least \(1-\alpha\)
\citep{hoeffding1963inequalities}.

Let \(i^\star\in\arg\min_i\mu_i\). On \(\mathcal E\), its lower endpoint is
at most \(\mu_{i^\star}\). Every active candidate \(j\) has upper endpoint at
least \(\mu_j\geq\mu_{i^\star}\). The lower endpoint of \(i^\star\) therefore
cannot exceed the smallest active upper endpoint. An optimal candidate
survives every look.

At the final look, choose an active candidate \(\widehat i\) with the smallest
empirical mean. Because an optimal candidate remains active,
\begin{equation}
\begin{aligned}
  \mu_{\widehat i}
  &\leq
  \widehat\mu_{\widehat i}+r(n_{\mathrm{final}})\\
  &\leq
  \widehat\mu_{i^\star}+r(n_{\mathrm{final}})\\
  &\leq
  \mu_{i^\star}+2r(n_{\mathrm{final}}).
\end{aligned}
\end{equation}
This proves both statements.
\end{proof}

For \(K=128\), \(J=4\), \(\alpha=0.05\), and
\(n_{\mathrm{final}}=1024\),
\begin{equation}
  2r(n_{\mathrm{final}})
  =
  0.1392446425.
\end{equation}
This is a marginal 95\% guarantee for one setting because the union bound in
Proposition~\ref{prop:finite-pool} covers the \(KJ\) intervals within that
setting. It is not a joint 95\% statement over the four experimental settings.
A prospective Bonferroni calibration for a four-setting family would use
\(\alpha/4=0.0125\) in Equation~\ref{eq:hoeffding-radius}, giving
\begin{equation}
  2r_{\mathrm{four\ settings}}(1024)
  =
  0.1486495094.
\end{equation}
The experiment used the per-setting value. It therefore makes only
the marginal claim. This guarantee concerns excess cost within the pool. It
does not say that the pool contains a global optimum, that the bound is tight,
or that the selected candidate beats an external comparator. That last
question requires independent confirmation.

The theory now identifies the quantities that an evaluation must expose:
surrogate error, selector error, coverage, comparison mass, margin, and
independent task loss. We next define the surrogate and selector families used
to test these links.

\section{Surrogates and Selectors Under Evaluation}

\subsection{Measured finite difference surrogates}

The SQuAD study evaluates single unit and pair removals exactly on the
calibration split. Let \(L_0=\bar{\loss}_{\mathrm{cal}}(0)\),
\(L_i\) be the empirical calibration loss after removing unit \(i\), and
\(L_{ij}\) the corresponding loss after removing \(i\) and \(j\). Define
\begin{equation}
  c_i=L_i-L_0,
  \qquad
  q_{ij}=L_{ij}-L_i-L_j+L_0.
\end{equation}
Let \(\mathbf Q\) be symmetric with
\begin{equation}
  Q_{ij}=Q_{ji}=q_{ij}\quad(i\neq j),
  \qquad Q_{ii}=0.
\end{equation}
For removal vector \(m\in\{0,1\}^d\), the measured quadratic surrogate is
\begin{equation}
  \widehat\Delta_{\mathrm M}(m)
  =
  c^\top m+\frac12m^\top \mathbf Qm.
  \label{eq:measured-quadratic}
\end{equation}
The folded score assigns half of every pair term to each incident unit:
\begin{equation}
  s_i
  =
  c_i+\frac12\sum_{j\neq i}q_{ij},
  \qquad
  \widehat\Delta_{\mathrm S}(m)
  =
  \sum_i s_i m_i.
\end{equation}
The singleton and pair finite differences are exact on the calibration split.
The approximation occurs when their quadratic expansion is used for an
arbitrary multiunit mask.

\subsection{Gradient and empirical second moment surrogates}

Let \(z\) be a vector of logical gates and let
\begin{equation}
  g_n
  =
  \left.\nabla_z\ell_n(z)\right|_{z=\mathbf 1}
\end{equation}
be the per example gate gradient. Define
\begin{equation}
  \mu=\frac1{N_{\mathrm{cal}}}\sum_{n=1}^{N_{\mathrm{cal}}} g_n,
  \qquad
  F=\frac1{N_{\mathrm{cal}}}
  \sum_{n=1}^{N_{\mathrm{cal}}}g_ng_n^\top.
\end{equation}
The matrix \(F\) is the raw empirical second moment of gate gradients. We do
not interpret it as the population Fisher information or an exact Hessian.
Removing units changes gates by \(\Delta z=-m\). The three QQP surrogates are
\begin{align}
  \widehat\Delta_{\mathrm{FO}}(m)
  &=
  -\mu^\top m,
  \label{eq:fo}\\
  \widehat\Delta_{\mathrm{DF}}(m)
  &=
  -\mu^\top m+\frac12\operatorname{diag}(F)^\top m,
  \label{eq:df}\\
  \widehat\Delta_{\mathrm{QF}}(m)
  &=
  -\mu^\top m+\frac12m^\top Fm.
  \label{eq:qf}
\end{align}
The three forms separate first order, diagonal second moment, and full
quadratic information. The diagonal form uses \(m_i^2=m_i\) for binary
removal variables.

\subsection{Greedy construction and deterministic one swap search}

For a symmetric pair matrix \(\mathbf Q\) with zero diagonal, write the
quadratic objective on selected removal set \(S\) as
\begin{equation}
  J(S)
  =
  \sum_{i\in S}c_i
  +
  \sum_{\{i,j\}\subseteq S}q_{ij}.
\end{equation}
Forward greedy construction adds the feasible unit with the smallest
increment until the budget is reached. The local selector starts from the
matching greedy set and enumerates every feasible exchange
\begin{equation}
  S'=(S\setminus\{i\})\cup\{j\},
  \qquad i\in S,\quad j\notin S.
\end{equation}
The exact objective change is
\begin{equation}
  \Delta(i\to j\mid S)
  =
  c_j-c_i
  +
  \sum_{u\in S\setminus\{i\}}
  (q_{ju}-q_{iu}).
  \label{eq:swap-delta}
\end{equation}
At iteration \(t\), the selector accepts the best feasible exchange satisfying
\begin{equation}
  \Delta(i\to j\mid S_t)<-\eta_t.
\end{equation}
The numerical tolerance is study specific. SQuAD uses the relative
scale
\begin{equation}
  \eta_t^{\mathrm{SQuAD}}
  =
  10^{-12}\max\{1,|J(S_t)|\},
\end{equation}
whereas QQP uses the absolute tolerance
\begin{equation}
  \eta_t^{\mathrm{QQP}}=10^{-12}.
\end{equation}
Both implementations enumerate exchanges deterministically. Exact ties follow
lexicographic order; the QQP implementation also treats objective values
within its absolute tolerance as tied. The symbol \(\eta_t\) is a numerical
acceptance tolerance, not the estimator error \(\varepsilon\) in Proposition
\ref{prop:approx-selector} and not a comparison-domain suboptimality
certificate \(\delta\).

\begin{proposition}[Finite termination and local optimality]
\label{prop:local-search}
Starting from any feasible fixed cardinality set, the one-swap selector
terminates after finitely many accepted exchanges. Every iterate is feasible.
At termination, no feasible one-for-one exchange decreases the objective by
more than the terminal tolerance.
\end{proposition}

\begin{proof}
Every accepted exchange preserves cardinality and is checked against the same
layer constraints, so every iterate remains in the finite feasible set. The
acceptance rule strictly decreases \(J\). A previously visited mask cannot be
revisited because that would require the objective to recover an earlier
value. A strictly descending sequence over a finite set must terminate. At
termination, exhaustive enumeration contains no feasible exchange with
\(\Delta<-\eta_t\). This is exactly the stated local optimality condition up
to the numerical tolerance.
\end{proof}

This proposition is intentionally local. No global or certified optimality
claim is made. Crossing the same greedy and local selectors with the measured
and empirical second moment objectives keeps estimator and selector effects
separately interpretable.

\subsection{Restricted reconstruction}

OSSCAR optimizes a support while refitting the retained weights
\citep{meng2024osscar}. For an activation matrix \(X\) with \(T\) calibration
positions and a transposed weight matrix \(B\), define
\begin{equation}
  H=\frac1T X^\top X,
  \qquad
  \mathcal D=\{i:H_{ii}=0\},
  \qquad
  (B_{\mathrm{eff}})_{i:}
  =
  \begin{cases}
    0,&i\in\mathcal D,\\
    B_{i:},&i\notin\mathcal D,
  \end{cases}
\end{equation}
and then define
\begin{equation}
  \qquad
  \widetilde H
  =H+\gamma\operatorname{mean}(\operatorname{diag}H)I,
  \qquad
  G=\widetilde H B_{\mathrm{eff}},
\end{equation}
where dead coordinates are handled before damping and the reference
implementation uses \(\gamma=0.01\). For a retained coordinate support \(S\),
its restricted objective is
\begin{equation}
  f(S)
  =
  \min_{W:\,W_{S^c}=0}
  \left\{
    \frac12\operatorname{tr}(W^\top\widetilde H W)
    -\operatorname{tr}(W^\top G)
  \right\}.
  \label{eq:osscar-restricted-objective}
\end{equation}
Positive definiteness from damping gives the unique retained solution
\begin{equation}
  W_S^*=\widetilde H_{SS}^{-1}G_S,
  \qquad W_{S^c}^*=0,
  \qquad
  f(S)=-\frac12\operatorname{tr}
  \left(G_S^\top\widetilde H_{SS}^{-1}G_S\right).
  \label{eq:osscar-restricted-solve}
\end{equation}
Because \(G=\widetilde H B_{\mathrm{eff}}\), the excess over the dense unconstrained
solution is
\begin{equation}
  f(S)-f(\mathrm{all})
  =
  \frac12\operatorname{tr}
  \left((W_S^*-B_{\mathrm{eff}})^\top
  \widetilde H(W_S^*-B_{\mathrm{eff}})\right).
  \label{eq:osscar-restricted-excess}
\end{equation}
Equations~\ref{eq:osscar-restricted-objective} to
\ref{eq:osscar-restricted-excess} define the reconstruction quantity used in
our OSSCAR evaluation. They include the retained-weight refit and reference
damping. OSSCAR's search is algorithmic, so these equations do not imply that
its returned support is globally optimal among all supports.

These evaluation subjects span measured finite differences, gradient second
moments, deterministic local search, and a public reconstruction method. The
experiments below use them to test four distinct transfers from surrogate
evidence to a pruning decision.

\section{Experimental Design}

\begin{figure}[!tbp]
\centering
\includegraphics[width=\textwidth]{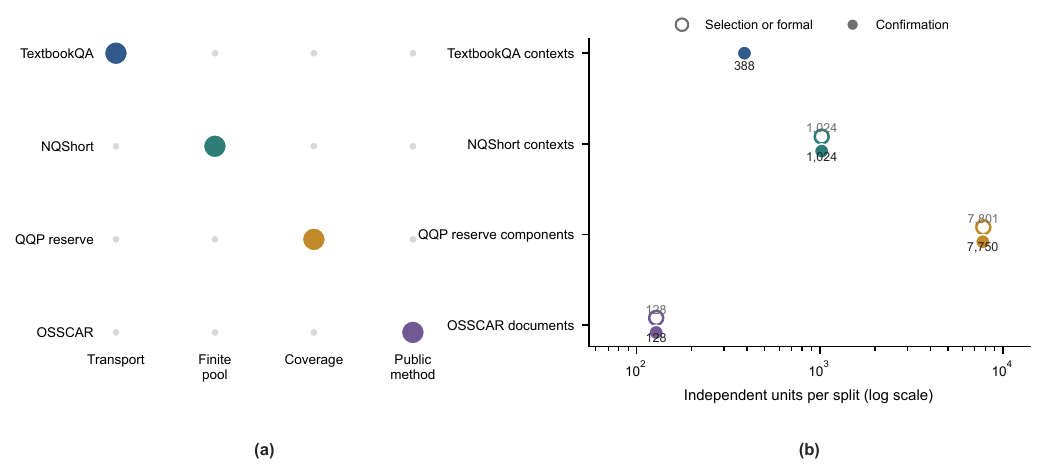}
\caption{\textbf{Four studies test distinct links between surrogate evidence
and a pruning decision.}
\textbf{(a)} TextbookQA tests external transport, Natural Questions tests a
fixed finite pool, QQP tests the coverage mechanism, and OPT-125M tests a
public reconstruction method. Large markers identify each study's primary
role.
\textbf{(b)} Independent units in the selection or primary split and in the
confirmation split. Open and filled markers distinguish the two roles.
TextbookQA shows only the 388 external confirmation contexts because its masks
and comparators were fixed on SQuAD. Units are passage contexts, connected
question components, or C4 documents as indicated on the vertical axis.}
\label{fig:study-architecture}
\end{figure}

Figure~\ref{fig:study-architecture} summarizes the four empirical studies and
their independent units.

\subsection{Common evaluation protocol}

A PruneShift study fixes five objects before confirmation: the feasible mask
set, surrogate, selector, comparison set, and statistical analysis. Primary
data may be used to fit a surrogate or select a decision. Confirmation data
cannot change that decision. Model outputs are reduced to the prespecified
experimental unit before inference.

The protocol distinguishes a scientific estimand from its implementation.
The following subsections define the data source, independent unit, estimand,
multiplicity family, and effect direction for each study. Encoder experiments
use task-tuned BERT-base and RoBERTa-base checkpoints
\citep{devlin2019bert,liu2019roberta}. We remove complete attention heads or
contiguous groups of 256 feed-forward channels. The decoder experiment uses
OPT-125M \citep{zhang2022opt}. It removes attention-head input blocks or
individual feed-forward neurons in selected layers.

\subsection{External TextbookQA confirmation}

The external study uses the text-only, answerable subset of TextbookQA from an
MRQA release \citep{kembhavi2017textbookqa,fisch2019mrqa}.
This subset is not the full multimodal TextbookQA benchmark. It contains 1,503
extractive QA examples in 389 passage contexts. Exact and normalized contexts
are disjoint from the SQuAD training and validation material used to construct
the masks \citep{rajpurkar2018squad2}.

The comparison inventory is fixed without using TextbookQA outcomes. One
comparator per setting is selected by the SQuAD analysis. It is therefore an
external comparator rather than a choice informed by TextbookQA. Five method
roles are computed from the fixed calibration and
surrogate inputs: the measured singleton role \(A\), empirical second-moment
greedy and local roles \(F\)-G and \(F\)-L, and measured finite-difference
greedy and local roles \(M\)-G and \(M\)-L. The local roles use the deterministic
one-swap selector in Proposition~\ref{prop:local-search}.

For QA example \(i\), the bounded cost is \(c_i=1-\mathrm{F1}_i\). Each
contrast is method cost minus the cost of its fixed comparator. Negative
values favor the method role. We first average QA costs within a passage
context and then average contexts equally. The primary analysis excludes one
prespecified near-duplicate context and uses 388 clusters. A sensitivity
analysis uses all 389.

The 20 contrasts form one multiplicity family. We draw 20,000 common
context-cluster bootstrap samples and use a studentized maximum absolute
statistic. The resulting intervals describe resampling stability for this
finite external collection. They are not distribution-free population
intervals.

\subsection{Natural Questions finite pool}

NaturalQuestionsShort comes from the same MRQA release and derives from the
Natural Questions benchmark \citep{kwiatkowski2019natural,fisch2019mrqa}. A
label-free content rule assigns complete passage contexts to three disjoint
sets. Selection and confirmation each contain 1,024 contexts. The remaining
contexts are untouched. All QA rows from one passage remain in the same set.

Each setting has a fixed pool of \(K=128\) candidates and \(J=4\) cumulative
looks. The observation for a candidate is its mean \(1-\mathrm{F1}\) within a
context, so every independent observation lies in \([0,1]\). Proposition
\ref{prop:finite-pool} therefore applies with \(B=1\). We report a
per-setting 95\% excess-cost bound and a separate four-setting Bonferroni
bound. The latter is required for a familywise statement.

Selection is completed before the confirmation roles are constructed. The
chosen role is then reproduced independently. Confirmation contrasts the
selected mask with one fixed comparator per setting. The primary interval is
a conservative finite-population Hoeffding interval without replacement. All four settings
are considered jointly.

\subsection{Controlled QQP coverage study}

The coverage study uses a held-out QQP reserve
\citep{wang2018glue}. Shared normalized questions induce connected components.
Components, not rows, are assigned to disjoint primary and confirmation
partitions. The primary partition contains 8,192 rows in 7,801 components. The
confirmation partition contains 8,192 rows in 7,750 components.

The design crosses six fixed checkpoints, two unit families, and two budgets.
This gives 24 cells. Each cell uses 13 prespecified mixture arms and 256 common
random designs. The arms vary the mass placed on a fixed comparison set while
leaving the estimator and selector definitions unchanged. Twelve adjacent
coverage contrasts per cell produce
\(24\times12\times256=73{,}728\) theory-calibration records.

For every record we compute the broad weighted squared error, exact density
ratio \(\kappa\), comparison mass \(q_{\min}\), comparison-set maximum error,
true decision margin, exact-selection indicator, and observed regret. The
selector optimization term is \(\delta=0\) in this exhaustive finite design.
We evaluate both the uniform and coverage inequalities from Section
\ref{sec:theory}.

The prespecified endpoints reduce the design to four architecture
and unit families. Within each family we study maximum surrogate error and
selected regret. For each outcome we report the weak-minus-strong coverage
contrast and the slope across ordered coverage levels. This gives 16
endpoints. Positive values support the coverage mechanism. Inference uses
20,000 common component-bootstrap samples and a maximum-statistic adjustment
across all endpoints.

\subsection{Restricted OSSCAR reconstruction}

The OSSCAR study uses the reference objective in Equations
\ref{eq:osscar-restricted-objective} to
\ref{eq:osscar-restricted-excess} \citep{meng2024osscar}. Every candidate
support is evaluated after solving the damped restricted system for retained
weights. Directly zeroing the original weights is not used as a substitute.

We study six OPT-125M layers, two unit families, and removal rates of 25\% and
50\%. This gives 24 cells. C4 validation documents are assigned without
overlap to 64 calibration documents, 128 primary documents, and 128
confirmation documents \citep{raffel2020exploring}. One fixed-length sequence
is taken from each document. The document is the inferential unit.

For each cell, 64 broad supports are matched to 64
selector-neighborhood supports. Primary fidelity compares local minus broad
absolute reconstruction mismatch for the mean, 90th percentile, and maximum.
Negative values indicate lower mismatch near the selector. The 72 cell
endpoints and three aggregate endpoints form one family.

OSSCAR and its task comparator are fixed before confirmation. The comparator
is selected from the declared primary comparison pool. It is not changed for
each confirmation document. The confirmation effect is
\begin{equation}
  \Delta_{\mathrm{task}}
  =
  \overline{\operatorname{NLL}}(m_{\mathrm{OSSCAR}})
  -
  \overline{\operatorname{NLL}}(m_{\mathrm{comp}}).
\end{equation}
Negative values favor OSSCAR. The 24 cell effects and one aggregate effect form
a second family. Both families use 20,000 common document-bootstrap samples
and simultaneous intervals.

\section{Results}

\subsection{External transport is heterogeneous}

Figure~\ref{fig:external-pool}(a) reports all 20 TextbookQA contrasts. Seven
simultaneous intervals lie below zero, six lie above zero, and seven cross
zero. The pattern differs by setting. BERT heads contain two method-favored
and three comparator-favored contrasts. All five BERT FFN contrasts are
inconclusive. RoBERTa heads again contain two method-favored and three
comparator-favored contrasts. RoBERTa FFN groups contain three method-favored
and two inconclusive contrasts.

\begin{figure}[!tbp]
\centering
\includegraphics[width=\textwidth]{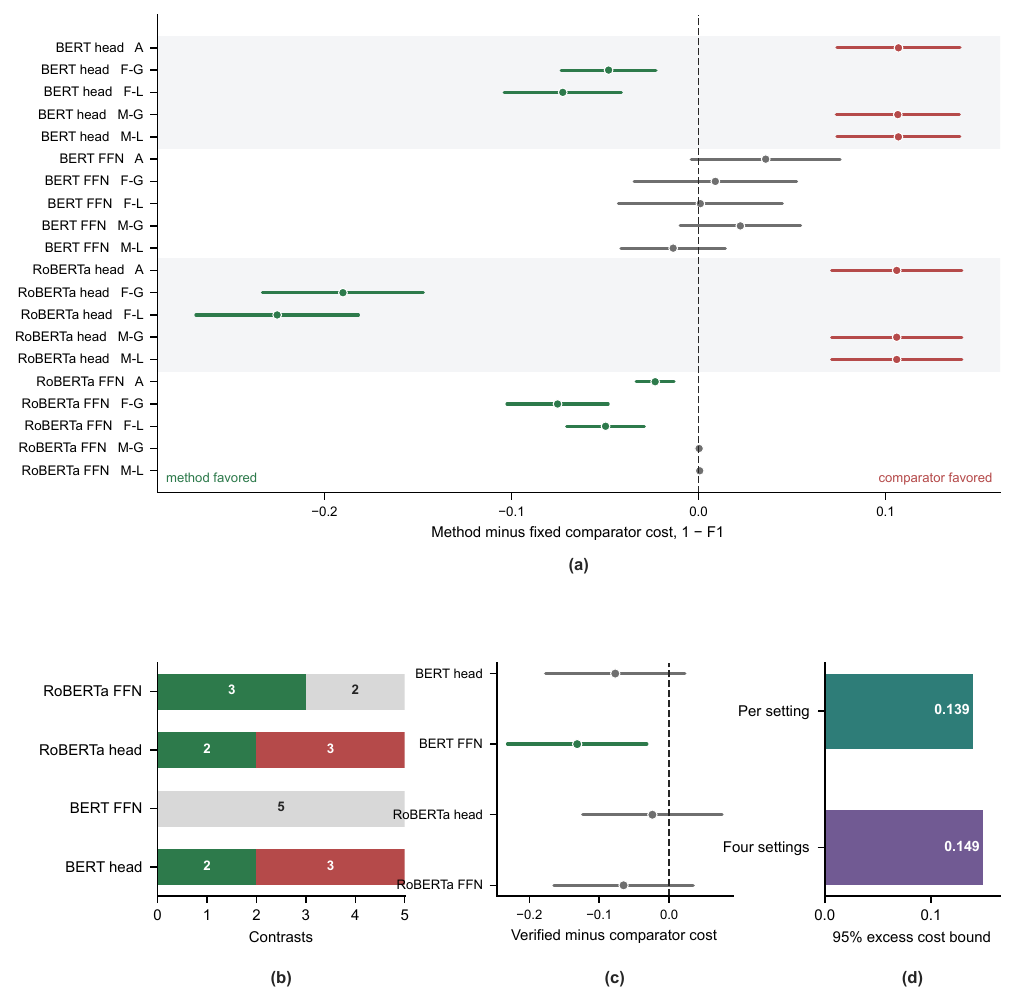}
\caption{\textbf{External confirmation is heterogeneous, and the finite-pool
result is scoped.}
\textbf{(a)} Method-minus-comparator cost on TextbookQA. Points are
equal-context estimates. Lines are simultaneous 95\% context-bootstrap
intervals across 20 fixed contrasts. Negative values favor the method role.
\textbf{(b)} Counts of interval directions by architecture and unit family.
\textbf{(c)} Selected-minus-comparator cost on 1,024 disjoint Natural Questions
confirmation contexts. Lines are strict familywise finite-population
intervals. Only BERT FFN lies below zero.
\textbf{(d)} The fixed-look pool certificate bounds excess cost by 0.139 within
one setting and by 0.149 across four settings. It does not certify superiority
over an external comparator.}
\label{fig:external-pool}
\end{figure}

The all-context sensitivity analysis preserves the qualitative conclusion.
The study therefore rejects a universal transport claim. It does not reject
every surrogate. Some fixed comparisons transfer strongly, while others
reverse or remain unresolved. The correct conclusion is conditional on
architecture, unit family, and comparison role.

\subsection{A valid finite-pool certificate does not imply universal improvement}

The Natural Questions certificate is numerically valid for the fixed pool.
The per-setting 95\% excess-cost bound is 0.1392. The four-setting familywise
bound is 0.1486. These values bound the selected candidate's excess cost
relative to the best member of the 128-candidate pool. They do not compare the
pool winner with a different pruning strategy.

Independent confirmation makes that second comparison. The BERT FFN estimate
is \(-0.1315\), with interval \([-0.2310,-0.0319]\). The other three intervals
cross zero: BERT heads give \(-0.0770\) with
\([-0.1765,0.0226]\), RoBERTa heads give \(-0.0239\) with
\([-0.1234,0.0757]\), and RoBERTa FFN groups give \(-0.0651\) with
\([-0.1647,0.0345]\). Strict improvement therefore holds in one of four
settings.

This result separates validity from usefulness. The certificate correctly
controls excess cost inside its declared pool. It does not guarantee that the
pool is strong enough to beat a fixed external comparator.

\subsection{Coverage changes decision error as predicted}

All 16 prespecified QQP confirmation intervals lie above zero
(Figure~\ref{fig:coverage}). Primary and confirmation signs agree for every
endpoint. We observe the same direction for maximum surrogate error and
selected regret, under both the extreme coverage contrast and the ordered
coverage slope. The pattern holds for BERT and RoBERTa, and for attention heads
and FFN groups.

\begin{figure}[!tbp]
\centering
\includegraphics[width=\textwidth]{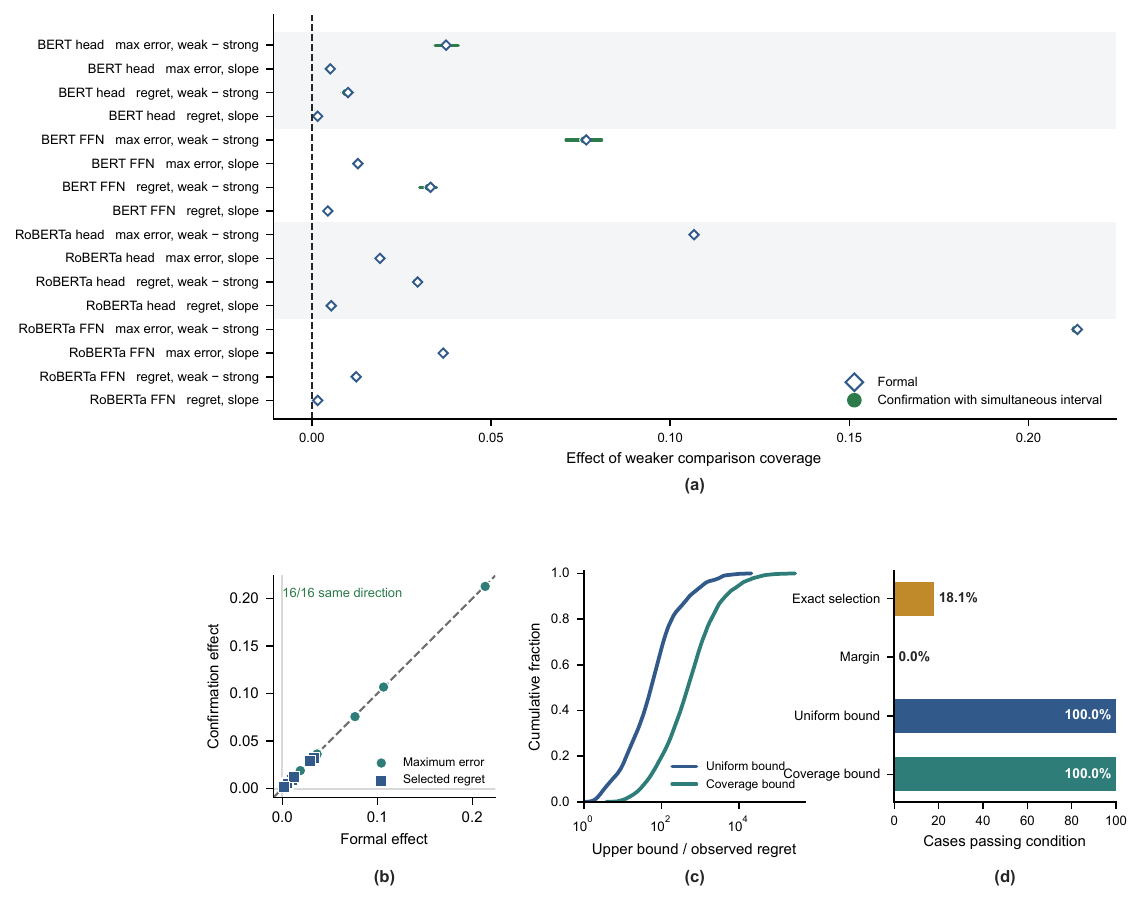}
\caption{\textbf{The controlled coverage intervention supports the proposed
mechanism, while the sufficient bounds remain conservative.}
\textbf{(a)} Primary estimates and confirmation estimates for 16 prespecified
coverage effects. Lines are simultaneous 95\% component-bootstrap intervals.
Positive values mean that weaker comparison coverage increases error or
regret.
\textbf{(b)} Primary and confirmation estimates for the same endpoints.
Circles denote maximum-error effects, squares denote selected-regret effects,
and the dashed line marks equality.
\textbf{(c)} Empirical distributions of the uniform and coverage upper bounds
divided by observed regret for the 60,409 positive-regret records. A ratio of
one would be tight.
\textbf{(d)} Exact selection occurs in 18.1\% of records. The margin condition
does not hold in this design. Both upper-bound inequalities hold in all
73,728 records.}
\label{fig:coverage}
\end{figure}

The record-level checks clarify what this positive mechanism result means.
The uniform inequality and coverage-transfer inequality both hold in
73,728 of 73,728 records. Exact surrogate selection occurs in 13,319 records,
or 18.1\%. The sufficient margin condition holds in none. The empirical bound
ratios are often much larger than one.

There is no contradiction. The inequalities are valid upper bounds. The
margin test is sufficient rather than necessary. The intervention shows that
coverage changes the observed error and regret in the predicted direction. It
also shows that the available sufficient conditions are too conservative to
explain most exact selections.

\subsection{OSSCAR improves local fidelity but does not win the fixed-mask test}

Restricted reconstruction produces a clear local fidelity pattern
(Figure~\ref{fig:osscar}). Among the 75 primary fidelity endpoints, 68
simultaneous intervals are negative, five cross zero, and two are positive.
The aggregate local-minus-broad effects are \(-0.0567\) for mean mismatch,
\(-0.0801\) for the 90th percentile, and \(-0.1302\) for the maximum. Their
simultaneous intervals are, respectively,
\([-0.0610,-0.0524]\), \([-0.0867,-0.0734]\), and
\([-0.1451,-0.1154]\).

\begin{figure}[!tbp]
\centering
\includegraphics[width=\textwidth]{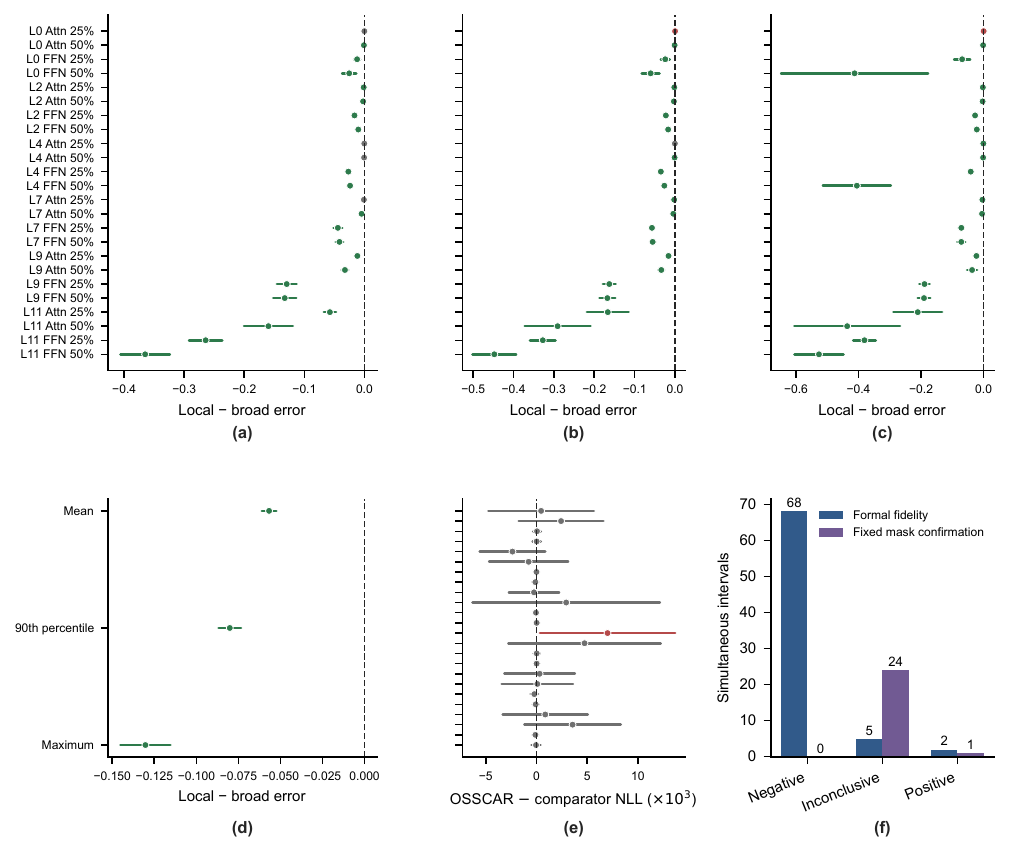}
\caption{\textbf{Restricted reconstruction improves local fidelity but
does not establish fixed-mask task superiority.}
\textbf{(a)}, \textbf{(b)}, and \textbf{(c)} Local-minus-broad reconstruction
mismatch in 24 OPT-125M cells.
Points show mean, 90th-percentile, or maximum mismatch contrasts. Lines are
simultaneous 95\% document-bootstrap intervals. Negative values indicate
better fidelity near OSSCAR.
\textbf{(d)} Aggregate primary fidelity effects.
\textbf{(e)} OSSCAR-minus-comparator NLL for 24 fixed confirmation contrasts.
Negative values favor OSSCAR.
\textbf{(f)} Direction counts for the 75 primary fidelity endpoints and 25
fixed-mask confirmation endpoints. The aggregate confirmation effect is
\(0.000774\), with interval \([-0.000300,0.001848]\).}
\label{fig:osscar}
\end{figure}
\FloatBarrier

The fixed-mask task comparison gives a different answer. Twenty-four of 25
simultaneous intervals cross zero. One cell favors the comparator. None favors
OSSCAR. The aggregate OSSCAR-minus-comparator NLL is \(0.000774\), with
interval \([-0.000300,0.001848]\).

The two findings are compatible. Local reconstruction fidelity is a property
of the proxy near the selector. Fixed-mask NLL is a task comparison between
two fixed supports. Better local fidelity does not, without an additional
link, prove that the selected support has lower task loss.

\subsection{Cross-study endpoint summary}

Table~\ref{tab:endpoint-summary} reports the directional outcome of each
prespecified endpoint family. It complements Figures~\ref{fig:external-pool}
to \ref{fig:osscar} by collecting exact counts without treating different
metrics as commensurate effect sizes.

\begin{table}[H]
\centering
\caption{Directional summary of simultaneous intervals. The prespecified
direction is lower method cost for TextbookQA and Natural Questions, a positive
weak-minus-strong coverage effect for QQP, lower local reconstruction mismatch
for OSSCAR fidelity, and lower OSSCAR NLL for fixed-mask confirmation.
``Aligned'' counts intervals in that direction, whereas ``Contains 0'' counts
intervals that include zero.}
\label{tab:endpoint-summary}
\small
\setlength{\tabcolsep}{5pt}
\begin{tabularx}{\textwidth}{@{}Yrrrrr@{}}
\toprule
Endpoint family & \(n\) & Aligned & Contains 0 & Opposite & Aligned (\%) \\
\midrule
TextbookQA external task cost
& 20 & 7 & 7 & 6 & 35.0 \\
Natural Questions fixed comparison
& 4 & 1 & 3 & 0 & 25.0 \\
QQP coverage effect
& 16 & 16 & 0 & 0 & 100.0 \\
OSSCAR local reconstruction fidelity
& 75 & 68 & 5 & 2 & 90.7 \\
OSSCAR fixed-mask task NLL
& 25 & 0 & 24 & 1 & 0.0 \\
\bottomrule
\end{tabularx}
\end{table}

\section{Discussion and Limitations}

PruneShift changes the unit of evaluation. A surrogate curve is evidence about
prediction on the masks used to draw that curve. A pruning decision is evidence
about one selected mask and its alternatives. The theory shows why the first
object does not automatically certify the second. The experiments show that
the missing transfer can succeed, fail, or remain unresolved.

Coverage is central. Broad evaluation can miss the thin region explored by a
selector. The density-ratio result makes this dependence explicit. The QQP
intervention then changes comparison coverage while holding the main
estimator and selector definitions fixed. Its result supports the mechanism,
but the loose bounds warn against treating a sufficient inequality as a sharp
performance predictor.

A finite-pool certificate offers a different route. It avoids a natural
density-ratio estimate and gives a transparent guarantee within a declared
pool. Its scope must remain visible. Natural Questions shows that a valid
within-pool guarantee can coexist with an inconclusive comparison against an
external strategy.

The OSSCAR study adds a second caution. Reconstruction is part of the pruning
intervention, so it must match the public method. Once the retained weights are
truly refitted, local fidelity improves strongly. The independent task
comparison remains unresolved. This is not a failure of the reconstruction
objective. It shows that proxy fidelity and task superiority are different
claims.

The framework is useful even when it produces a null result. TextbookQA rules
out a simple universal transport story. Natural Questions narrows the
finite-pool claim. OSSCAR separates a strong proxy result from an unsupported
task claim. These distinctions make subsequent algorithm development more
targeted. They identify whether a method needs a better estimator, better
coverage, a stronger selector, or a direct confirmation design.

Several limitations define the scope of these conclusions. Encoder evidence
uses BERT and RoBERTa on
reading comprehension and paraphrase classification. The external
TextbookQA source is an answerable, text-only subset of a multimodal dataset.
The decoder study uses one model scale and six OPT-125M layers. Larger models,
other tasks, and other structured pruning methods may produce different
transfer patterns.

The comparison domains are declared finite sets. Regret in such a set is a
lower bound on global regret when the selected mask belongs to the set, as
shown in Lemma~\ref{lem:pool-lower}. Its magnitude need not approximate the
global gap. The Natural Questions certificate is therefore a pool guarantee,
not a global pruning guarantee.

Some intervals have deliberately limited interpretations. TextbookQA
bootstrap intervals describe stability across the observed context clusters.
They are not distribution-free population intervals. The QQP checkpoints are
fixed design instances rather than a random sample from a population of
training seeds.

The theoretical guarantees are sufficient and can be conservative. The QQP
study confirms both validity and looseness. Better data-dependent bounds are a
promising extension, but they must preserve independence and account for
selector adaptation.

Finally, all experiments use logical masks. They do not measure exported model
latency, throughput, energy, recovery training, or hardware speedup. The work
evaluates pruning decisions, not deployment efficiency.

\section{Conclusion}

Structured pruning surrogates are used to make decisions. Their evaluation
should therefore reach the decision itself. PruneShift separates broad
predictive fidelity, fidelity near selector outputs, and fixed comparison
quality. It proves that aggregate rank agreement cannot certify an argmin
without additional assumptions. It then provides coverage, margin, finite
comparison, and finite-pool conditions that make the missing assumptions
explicit.

The four studies support this scoped view. External transport is
heterogeneous. A valid finite-pool certificate does not guarantee universal
improvement. Controlled coverage changes error and regret in the predicted
direction, but the bounds are loose. Restricted OSSCAR reconstruction improves local
fidelity without establishing fixed-mask task superiority. Together, these
results provide a clearer standard for evaluating the decisions made by
structured pruning surrogates.

{\fontsize{9.5}{11}\selectfont
\bibliographystyle{unsrtnat}
\bibliography{references}
}

\end{document}